\def\year{2020}\relax
\documentclass[letterpaper]{article} % DO NOT CHANGE THIS
\usepackage{aaai20}  % DO NOT CHANGE THIS
\usepackage{times}  % DO NOT CHANGE THIS
\usepackage{helvet} % DO NOT CHANGE THIS
\usepackage{courier}  % DO NOT CHANGE THIS
\usepackage{url}  % DO NOT CHANGE THIS
\usepackage{graphicx} % DO NOT CHANGE THIS
\usepackage[ruled]{algorithm2e} 
\usepackage{subfig}
\usepackage{amsthm}
\usepackage{algorithmic}
\usepackage{amsmath}
\usepackage{amssymb}
\usepackage{multirow} 
\newtheorem{claim}{Claim}
\usepackage[utf8x]{inputenc}
\usepackage{xcolor}
\DeclareUnicodeCharacter{ +65306}{~}
\usepackage{graphicx}  % DO NOT CHANGE THIS
\newtheorem{prop}{Proposition}
\title{Large-scale Multi-view Subspace Clustering in Linear Time}
\author{  \Large \textbf{Zhao Kang,\textsuperscript{\rm 1} Wangtao Zhou,\textsuperscript{\rm 1} Zhitong Zhao,\textsuperscript{\rm 1} Junming Shao,\textsuperscript{\rm 1} Meng Han,\textsuperscript{\rm 2}Zenglin Xu,\textsuperscript{\rm 1}\textsuperscript{\rm ,3}}\\ % All authors must be in the same font size and format. Use \Large and \textbf to achieve this result when breaking a line
\textsuperscript{\rm 1}School of Computer Science and Engineering, University of Electronic Science and Technology of China, China\\ %If you have multiple authors and multiple affiliations
 \textsuperscript{\rm 2}School of Management and Economics, University of Electronic Science and Technology of China, China\\
  \textsuperscript{\rm 3}Centre for Artificial Intelligence, Peng Cheng Lab, Shenzhen 518055, China\\
\{zkang,zlxu\}@uestc.edu.cn% email address must be in roman text type, not monospace or sans serif
}
 
\begin{document}

\maketitle

\begin{abstract}
%Similarity measurement is a fundamental problem in numerous data-driven applications.
A plethora of multi-view subspace clustering (MVSC) methods have been proposed over the past few years. Researchers manage to boost clustering accuracy from different points of view. However, many state-of-the-art MVSC algorithms, typically have a quadratic or even cubic complexity, are inefficient and inherently difficult to apply at large scales. In the era of big data, the computational issue becomes critical. To fill this gap, we propose a large-scale MVSC (LMVSC) algorithm with linear order complexity. Inspired by the idea of anchor graph, we first learn a smaller graph for each view. Then, a novel approach is designed to integrate those graphs so that we can implement spectral clustering on a smaller graph. Interestingly, it turns out that our model also applies to single-view scenario. Extensive experiments on various large-scale benchmark data sets validate the effectiveness and efficiency of our approach with respect to state-of-the-art clustering methods. 
\end{abstract}

\section{Introduction}
%Clustering is a very popular technique to extract and identify underlying patterns in data \cite{jain1999data}. 
During the last decade, subspace clustering has been explored extensively due to its promising performance \cite{huang2019deep,peng2017subspace}. Basically, it is based on the assumption that the data points are drawn from multiple low-dimensional subspaces and each group fits into one of the low-dimensional subspaces. Significant progress has been made to uncover such underlying subspaces. 

Among the various subspace clustering approaches, spectral clustering based methods such as low-rank representation (LRR) \cite{liu2013robust,kang2015robust} and sparse subspace clustering (SSC) \cite{elhamifar2013sparse} have achieved great success. SSC finds a sparse representation of data points while LRR finds the low-rank structure of subspace. Afterwards, the spectral clustering algorithm is performed on such new representation \cite{ng2002spectral}. For $n$ data points, the first step typically takes $\mathcal{O}(n^2)$ or $\mathcal{O}(n^3)$ time while the second step requires $\mathcal{O}(n^3)$ or at least $\mathcal{O}(n^2k)$ time, where $k$ is the number of clusters. Therefore, these two time comsuming steps debar the application of subspace clustering from large-scale data sets.

Recently, much work have been developed to accelerate subspace clustering. For example, \cite{wang2014exact} proposes a data selection approach; a scalable SSC based on orthogonal matching pursuit is developed in \cite{you2016scalable}; a fast solver for SSC is proposed in \cite{pourkamali2018efficient}; accelerated LRR is also developed \cite{fan2018accelerated,xiao2015falrr}; a sampling approach is used in 
\cite{li2017large}. Unfortunately, they all target for single-view scenario and can not deal with multi-view data.

With the advance of information technology, many practical data come from multiple channels or appear in muliple modalities, i.e., the same object can be described in different kinds of features \cite{kang2019multi,yin2018multiview,zhan2018graph,liu2019multiple,multi2019yao,tang2019cross}. Different feature sets often characterize different yet complementary information. For instance, an image can have heterogeneous features, such as Garbor, HoG, SIFT, GIST, and LBP; an artical can be written in different languages. Accordingly, it is crucial to develop multi-view learning method to incorporate the useful informaiton from different views. 

There has been growing interest in developing multi-view subspace clustering (MVSC) algorithms. For example, MVSC \cite{gao2015multi} method learns a graph representation for each view and all graphs are assumed to share a unique cluster matrix; \cite{cao2015diversity} applies a diversity term to explore the complementarity information; \cite{luo2018consistent} explores both consistency and specificity of different views; \cite{wang2019multi,zhang2017latent} perform MVSC in latent representation; \cite{brbic2018} imposes both low-rank and sparsity constraints on the graph; \cite{kang2019multiple,kang2019partition} learn one partition for each view and them perform a fusion in partition space. In terms of clustering accuracy, these methods are attractive. Nevertheless, their computation efficiency, at least $\mathcal{O}(n^2k)$ complexity, is largely ignored. In the era of big data, it has been witnessed that many applications involve large-scale multi-view data. Therefore, their severely unaffordable complexity precludes MVSC from real-world applications with big data. Moreover, existing single-view subspace clustering acceleration techniques do not apply to multi-view scenarios owing to the heterogeneous nature of multi-view data.

To fill the gap, we need to address two challenges. First, how to avoid constructing $n\times n$ graph for multi-view data, which requires a lot of time and storage. Second, how to circumvent the computationally expensive step-----spectral clustering with $\mathcal{O}(n^2k)$ complexity. To this end, we introduce a novel large-scale multi-view subspace clustering (LMVSC) method, which can greatly reduce the computational complexity as well as improve the clustering performance. First, to address the large-scale graph construction problem, we select a small number of the instances as anchors, and then a smaller graph is built for each view. Second, a novel way to integrate mutiple graphs is designed so that we can implement spectral clustering on a small graph.

 The contributions of this paper mainly include：
\begin{itemize}
\item We make the first attempt towards large-scale multi-view subspace clustering. Our proposed method can be solved in linear time, while existing methods demand square or cubic time in the number of data instance.
\item Extensive experiments in terms of data sets and comparison methods validate the effectiveness and efficiency of the proposed algorithm.
 \item As a special case, our method also applies to single-view subspace clustering task. It gives promising performance on very large-scale data.
\end{itemize}

\section{Preliminaries}
%Our work relies on multi-view subspace clustering and anchor graph. We describe them in this section.
\subsection{Multi-view Subspace Clustering}
Subspace clustering expresses each data point as a linear combination of other points and minimizes the reconstruction loss to obtain the combination coefficient. This coefficient is treated as the similarity between corresponding points. Mathematically, given data $X\in\mathcal{R}^{d\times n}$ with $d$ features and $n$ samples, subspace clustering solves the following problem \cite{liu2013robust,elhamifar2013sparse,kang2019low}:
\begin{equation}
\min_Z \|X-XS\|_F^2+\alpha f(S)\quad s.t.\quad S\geq 0,\hspace{0.1cm}  S\textbf{1}=\textbf{1},
\end{equation}
where $f(\cdot)$ is a certain regularizer function, $\alpha>0$ is a balance parameter, and \textbf{1} is a vector with all ones. The constraints require that all $S_{ij}$ is nonnegative and $\sum_j S_{ij}=1$. $S$ is often called similarity graph matrix. It is obvious that $S$ has size of $n\times n$, which poses a big challenge with scalability to large-scale data sets.
%For simplicity, we choose Frobenius norm for $f$.

Recently, multi-view subspace clustering (MVSC) has attracted much attention. Basically, for multi-view data $X=[X^1; \cdots; X^i; \cdots; X^v]\in\mathcal{R}^{\sum\limits_{i=1}^v d_i\times n}$, MVSC \cite{gao2015multi,cao2015diversity,kang2019multi} solves:
\begin{equation}
\min_{S^i} \sum\limits_{i=1}^v \|X^i-X^iS^i\|_F^2+\alpha f(S^i)\hspace{0.15cm} s.t.\hspace{0.15cm}S^i\geq 0,\hspace{0.1cm}  S^i\textbf{1}=\textbf{1}.
\label{mvsc}
\end{equation}
With different forms of $f$, Eq. (\ref{mvsc}) gives solutions with different properties. For example, \cite{wang2016iterative} encourages similarity between pairs of graphs; \cite{cao2015diversity} emphasizes the complementarity. With these multiple graphs, \cite{gao2015multi} assumes they produce the same clustering; \cite{cao2015diversity,wang2016iterative} perform spectral clustering on averaged graph. All these MVSC methods take at least $\mathcal{O}(n^2k)$ time, hence they are generally computational prohibtive.

\subsection{Anchor Graph}
Anchor graph is an efficient way to construct the large-scale graph matrix  \cite{chen2011large,liu2010large,wang2016scalable,han2017orthogonal}. Specifically, a subset of $m$ points $(m\ll n)$ $A=[a_1, a_2,\cdots, a_m]\in\mathcal{R}^{d\times m}$ are chosen based on a $k$-means clustering or random sampling on the original data set. Each point plays a role as a landmark which approximates the neighborhood structure. Then, a sparse affinity matrix $Z\in\mathcal{R}^{n\times m}$ is constructed between the landmarks and the original data as follows:
\begin{equation}
Z_{ij}=
 \begin{cases}\frac{K_\delta(x_i,a_j)}{\sum_{j'\in<i>}K_\delta(x_i,a_{j'})}, &j\in<i>\\
   0, &\text{otherwise}
    \end{cases}
    \label{hand}
\end{equation}
where $<i>$ indicates the indexes of $r$ $ (r<m)$ nearest anchors around $x_i$. The function $K_\delta (\cdot)$ is the commonly used Gaussian kernel with bandwidth parameter $\delta$. Subsequently, the doubly-stochastic similarity matrix $S\in\mathcal{R}^{n\times n}$, the summation of each column and each row equal to one, is built based on:
\begin{equation}
S=\hat{Z}\hat{Z}^\top\quad \textit{where}\quad  \hat{Z}=Z\Sigma^{-1/2},
\label{large}
\end{equation}
where $\Sigma$ is a diagonal matrix with the entry $\Sigma_{ii}=\sum_{j=1}^n Z_{ji}$. 

We can see that the above graph construction approach is extremely efficient since only $\mathcal{O}(mn)$ distances are considered. However, this heuristic strategy has one inherent drawback which is always ignored, i.e., quality of the built graph heavily depends on the exponential function. In general, it might not be appropriate to the structure of the data space \cite{shakhnarovich2005learning,kang2019cyber}. Furthermore, it is also challenging to find a good kernel parameter without any supervision information. Consequently, the downstream task might not be able to obtain a good performance. 

To obtain a similarity measure tailored to the problem at hand, a principled way is to learn it from data \cite{kang2019Clustering}. In this paper, inspired by the idea of anchor graph, we use a smaller graph $Z$ to approximate the full $n\times n$ graph matrix. Different from existing approach, we adaptively learn $Z$ from raw data by solving an optimization problem.

\section{Methodology}
For large-scale data, there is much redundancy in the data; a small number of the instances are sufficient for reconstructing the underlying subspaces. 
Similar to anchor strategy, we use a smaller matrix $Z^i\in \mathcal{R}^{n\times m} ( m\ll n)$ to approximate the full $n\times n$ matrix $S^i$. Specifically, a set of $m$ landmarks $A^i=\{a_j\}_{j\in[1,m]}$ (cluster centers) are obtained through $k$-means 
clustering on the data $X^i$. Then, $Z^i$ is built between the landmarks $A^i$ and $X^i$. 

%, which approximate the neighborhood structure
Rather than using the handcrafted similarity measure (\ref{hand}), we learn $Z^i$ from data. For multi-view data, we solve the following problem:
\begin{equation}
\begin{split}
\min_{Z^i} &\sum_{i=1}^{v}\|X^i-A^i(Z^i)^\top\|_F^2+\alpha \|Z^i\|_F^2\\
&s.t.\hspace{.15cm} 0\leq Z^i,\hspace{.1cm} (Z^i)^\top\textbf{1}=\textbf{1}.
\end{split}
\label{sc}
\end{equation}
Compared to Eq. (\ref{mvsc}), we only consider the similarities between $mn$ points instead of $n^2$, so the complexity is remarkably reduced. Problem (\ref{sc}) can be easily solved via convex quadratic programming.

Given a set of $v$ graphs $\{Z^i\}_{i\in[1,v]}$, our goal is to merge them so as to make use of the rich information contained in $v$ views. Obviously, it is not feasible to run spectral clustering using $\{Z^i\}_{i\in[1,v]}$ since their size is not $n\times n$ anymore. For multi-view case, we can follow a similar procedure as the traditional method, i.e., apply Eq. (\ref{large}) to restore the $n\times n$ graph $S^i$ and then take average, 
\begin{equation}
\bar{S}=\frac{\sum_i S^i}{v}. 
\end{equation}
In the sequel, spectral clustering is implemented to achieve embedding $Q\in\mathcal{R}^{n\times k}$, i.e.,
\begin{equation}
\max_Q Tr(Q^\top\bar{S}Q)\quad  s.t.\quad Q^\top Q=\mathbf{I}.
\end{equation}
Its solution is the $k$ eigenvectors corresponding to the largest $k$ eigenvalues of $\bar{S}$. However, the eigen decomposition takes at least $\mathcal{O}(n^2k)$ time.

In this paper, relying on proposition \ref{th} \cite{chen2011large,affeldt2019spectral}, we present an alternative approach to compute the $k$ left singular vectors of the concatenated matrix $\bar{Z}\in\mathcal{R}^{n\times mv}$,
\begin{equation}
\bar{Z}=\frac{1}{\sqrt{v}}[\hat{Z}^1, \cdots, \hat{Z}^i, \cdots, \hat{Z}^v].\label{cont}
\end{equation}
Since $mv\ll n$, using $\bar{Z}$ instead of $\bar{S}$ naturally reduces the computational cost of $Q$.%The SVD of $Z$ leads to $U\Sigma V^T$. Then, we run k-means on $U$ to get the final clustering. 
\begin{prop}\label{th}
Given a set of similarity matrices $\{S^i\}_{i\in[1,v]}$, each of them can be decomposed as $Z^i(Z^i)^\top$. Define $\bar{Z}=\frac{1}{\sqrt{v}}[Z^1, \cdots, Z^i, \cdots, Z^v]\in\mathcal{R}^{n\times mv}$, its Singular Value Decomposition (SVD) is $U\Lambda V^\top$ (with $U^{\top}U=\mathbf{I}$, $V^{\top}V=\mathbf{I}$). We have,
\begin{equation}\label{equiveq}
\max_{Q^{\top}Q=\mathbf{I}}Tr(Q^{\top} \bar{S} Q) \Leftrightarrow \min_{Q^{\top}Q=\mathbf{I},H}||\bar{Z}-Q H^{\top}||_F^2.
\end{equation}\\
And, the optimal solution $Q^*$ is equal to $U$.
\end{prop}
\begin{proof}
Based on the second term of Eq.~
(\ref{equiveq}), one can easily show that $H^{*}=\bar{Z}^{\top}Q$. Plugging the value of $H^{*}$ into Eq.~(\ref{equiveq}), the following equivalences hold
\begin{eqnarray*}
\min_{Q^{\top}Q=\mathbf{I},H}||\bar{Z}-Q H^{\top}||_F^2 &\Leftrightarrow& \min_{Q^{\top}Q=\mathbf{I}}||\bar{Z}-QQ^{\top} \bar{Z}||_F^2\\
&\Leftrightarrow& \max_{Q^{\top}Q=\mathbf{I}}Tr(Q^{\top} \bar{Z}\bar{Z}^{\top} Q)\\
&\Leftrightarrow& \max_{Q^{\top}Q=\mathbf{I}}Tr(Q^{\top} \bar{S}Q).
\end{eqnarray*}

Moreover, 
\begin{eqnarray*}
\bar{S}=\bar{Z}\bar{Z}^{\top}&=&(U \Lambda V^{\top})(U \Lambda V^{\top})^\top\nonumber\\
&=&U \Lambda (V^{\top}V) \Lambda U^{\top}\nonumber\\
&=&U \Lambda^2 U^{\top}.
\end{eqnarray*}
Thereby the left singular vectors of $\bar{Z}$ are the same as the eigenvectors of $\bar{S}$. 
\end{proof}

\begin{algorithm}[htbp]%[!tb]
\caption{LMVSC algorithm}
\label{alg1}
\small
 {\bfseries Input:} Multi-view data matrix $X^1, \cdots, X^v$, cluster number $k$, parameter $\alpha$, anchors $\{A^i\}_{i\in[1,v]}$.\\

 {\bfseries Step 1:} Solve problem (\ref{sc});\\
 {\bfseries Step 2:} Construct $\bar{Z}$ as (\ref{cont});\\
 {\bfseries Step 3:}  Compute $Q$ by performing SVD on $\bar{Z}$.\\
{\bfseries Output:} Run $k$-means on $Q$ to achieve final clustering.\\
\end{algorithm}

The complete procedures for our LMVSC method is summarized in Algorithm \ref{alg1}. 

\begin{claim}
As a special case, the proposed LMVSC algorithm is also suitable to single-view data.
\end{claim}
\begin{proof}
It is easy to observe that our proposed precedure also works for single view data, i.e., $v=1$ and $\bar{Z}=\hat{Z}$ in Eq. (\ref{cont}).
\end{proof}

\subsection{Complexity Analysis}
LMVSC benefits from the low complexity of the anchor strategy for the graph construction and the eigen decomposition on a low-dimensional matrix. Specifically, the construction of $Z^is$ takes $\mathcal{O}(nm^3v)$, where $n$ is the number of data points, $m$ is the number of anchors $(m\ll n)$, $v$ is the view number. In experiments, we apply the built-in matlab function \texttt{quadprog} to solve (\ref{sc}), in which inter-point method with complexity $\mathcal{O}(m^3)$ is used \cite{ye1989extension}. It is worth mentioning that the construction of $Z^is$ can be easily parallelized over multiple cores, leading to more efficient computation. The computational cost for the $Q$ embedding induces a computational complexity of $\mathcal{O}(m^3v^3+2mvn)$. In addition, we need additional $\mathcal{O}(nmtp)$ for the $k$-means at the begining for anchor selection and $\mathcal{O}(nk^2t)$ for the last $k$-means on $Q$, where $t$ is the number of iterations, $k$ is the number of clusters, $p=\sum\limits_{i=1}^v d_i$ is the summation of feature dimensions. Accordingly, our proposed LMVSC method costs time only linear in $n$.

  \begin{table*}[!hbtp]
\begin{center}
%\centering
\caption{Description of the data sets. The dimension of features is shown in parenthesis. }
\label{datasets} 
\resizebox{1.8\columnwidth}{!}{
\begin{tabular}{lllll}
\hline%{1.0pt}%
%\hline

{View} &{Handwritten} & {Caltech7/20}  & {Reuters} & {NUS}  \\\hline
1& Profile correlations (216) & Gabor(48) & English (21531) & Color Histogram (65)\\
2& Fourier coefficients (76) & Wavelet moments (40)& French (24892)& Color moments (226)\\
3& Karhunen coefficients (64)  &CENTRIST (254)& German (34251)& Color correlation (145)\\
4 &  Morphological (6) &HOG (1984)& Italian (15506)& Edge distribution (74)\\
5& Pixel averages (240)  & GIST (512)& Spanish (11547)& Wavelet texture (129) \\
6& Zernike moments (47)  &LBP (928)& - &- \\\hline
Data points & 2000 & 1474/2386&18758 & 30000\\
Class number& 10 & 7/20 & 6 & 31\\
\hline
\end{tabular}}
\end{center}
\end{table*}

\begin{table}			
	\centering
		\caption{Clustering performance on Caltech7 data. “-” denotes some unreasonable values that are close to zero.\label{cal7}}
		\resizebox{.9\columnwidth}{!}{
		
				\begin{tabular}{|c |c | c|c| c|  }
					\hline
					Method& Acc& NMI& Purity& Time (s)\\
					\hline	
E3SC(1)&	0.3060&	0.1993&	0.3677&	10.32\\
E3SC(2)&	0.3256&	0.2648&	0.3582&	9.84\\
E3SC(3)&	0.4261&	0.2534&	0.4573&	18.19\\
E3SC(4)&	0.5509&	0.4064&	0.5787&	82.43\\
E3SC(5)&	0.5122&	0.3863&	0.5598&	27.80\\
E3SC(6)&	0.5244&	0.4273&	0.5773&	42.86\\
SSCOMP(1)&0.2619&	0.0640&	0.2958&	0.84\\
SSCOMP(2)&	0.3759&0.2017&0.4023&0.82\\
SSCOMP(3)&0.3256&0.1354&0.3460&1.50\\
SSCOMP(4)&0.5604&0.4631&0.6357&5.35\\
SSCOMP(5)&0.4715&-&-&2.11\\
SSCOMP(6)&0.5000&0.3688&0.5577&3.00\\						
AMGL&	0.4518&	0.4243&	0.4674&	20.12\\
MLRSSC&	0.3731&	0.2111&	0.4145&	22.26\\
MSC\_IAS&0.3976&0.2455&0.4444&57.18\\
LMVSC&	\textbf{0.7266}&	\textbf{0.5193}&	\textbf{0.7517}&	135.79\\

					\hline			
			\end{tabular}}
	
	%\end{center}}

	\end{table}

\section{Experiment on Multi-View Data}
In this section, we conduct extensive experiments to assess the performance of the proposed method on real-world data sets.
\subsection{Data Sets}
We perform experiments on several benchmark data sets: Handwritten, Caltech-101, Reuters, NUS-WIDE-Object. Specifically, Handwritten consists of handwritten digits of 0 to 9 from UCI machine learning repository. Caltech-101 is a data set of images for object recognition. Following previous work, two subsets, caltech7 and caltech20, are used. Reuters contains documets written in five different languages and their translations. A subset written in English and all their translations are used here. NUS-WIDE-Object (NUS) is also a object recognition database. More detailed information about these data is shown in Table \ref{datasets}.

\begin{table}			

\centering
	\caption{Clustering performance on Caltech20 data.\label{cal20}}
\resizebox{.9\columnwidth}{!}{
	%	\begin{center}
			
			\begin{tabular}{|c |c | c|c| c|  }
				\hline
				Method& Acc& NMI& Purity& Time (s)\\
				\hline	
				E3SC(1)&	0.2531&	0.2670&	0.2657&	27.71\\
				E3SC(2)&	0.2670&	0.3248&	0.2993&	26.81\\
				E3SC(3)&	0.2921&	0.3358&	0.3282&	46.69\\
				E3SC(4)&	0.4107&	\textbf{0.5576}&	0.4908&	180.70\\
				E3SC(5)&	0.4162&	0.4834&	0.4644&	61.32\\
				E3SC(6)&	0.4845&	0.4976&	0.5444&	99.62\\
SSCOMP(1)&0.1928&0.1359&0.2087&1.66\\
				SSCOMP(2)&0.2511&0.2338&0.2955&1.56\\
				SSCOMP(3)&0.1945&0.1319&0.2360&2.75\\
				SSCOMP(4)&0.3667&0.4487&0.4288&10.02\\
				SSCOMP(5)&0.2930&0.1681&\textbf{0.6010}&4.45\\
				SSCOMP(6)&0.3567&0.3735&0.4346&5.88\\
				AMGL&	0.3013&	0.4054&	0.3164&	77.63\\
				MLRSSC&	0.2821&	0.2670&	0.3039&	607.28\\
				MSC\_IAS&0.3127&0.3138&0.3374&93.87\\
				LMVSC&	\textbf{0.5306}&	0.5271&	0.5847&	342.97\\

				\hline			
			\end{tabular}
			
			}
	
\end{table}

\subsection{Clustering Evaluation}
We compare LMVSC\footnote{Code is available at https://github.com/sckangz/LMVSC} with several recent work as stated bellow:\\
\textbf{Efficient Solvers for Sparse Subspace Clustering (E3SC)}: recent subspace clustering method proposed in \cite{pourkamali2018efficient}. We run E3SC on each view and report its performance E3SC($\cdot$) .\\
\textbf{Scalable Sparse Subspace Clustering by Orthogonal Matching Pursuit (SSCOMP)}: one of the representative large-scale subspace clustering method developed in \cite{you2016scalable}. We run it on each single view.\\
\textbf{Parameter-Free Auto-Weighted Multiple Graph Learning (AMGL)}: one of the state-of-the-art multi-view spectral clustering method proposed in \cite{nie2016parameter}.\\
\textbf{Multi-view Low-rank Sparse Subspace Clustering (MLRSSC)}: recent method \cite{brbic2018} which outperforms many multi-view spectral clustering methods, e.g., \cite{xia2014robust}.\\
\textbf{Multi-view Subspace Clustering with Intactness-aware Similarity (MSC\_IAS)}: another recent multi-view subspace clustering method proposed in \cite{wang2019multi}, which reports improvements against a number of MVSC methods, e.g., \cite{zhang2017latent}.

All our experiments are implemented on a computer with a 2.6GHz Intel Xeon CPU and 64GB RAM, Matlab R2016a. For fair comparison, we follow the experimental setting described in respective papers and tune the parameters to obtain the best performance. We search anchor number $m$ in the range $[k, 50, 100]$. The motivation behind this is that the number of points required for revealing the underlying subspaces should not be less than the number of subspaces, i.e., the cluster number $k$. In addition, some other factors, such as subspace's intrinsic dimensions, noise level could also influence the anchor number. We select $\alpha$ from the range $[0.001, 0.01, 0.1, 1, 10]$. We assess the performance with accuracy (Acc), normalized mutual information (NMI), purity, and running time. 

\begin{table}			

	\centering
			\caption{Clustering performance on Handwritten data.\label{hw}}
			\resizebox{.9\columnwidth}{!}{
			\begin{tabular}{|c |c | c|c| c|  }
				\hline
				Method& Acc& NMI& Purity& Time (s)\\
				\hline	
				E3SC(1)&	0.7555&	0.7491&	0.8385&	29.94\\
				E3SC(2)&	0.5790&	0.5279&	0.6200&	20.80\\
				E3SC(3)&	0.7300&	0.7226&	0.8095&	19.47\\
				E3SC(4)&	0.4530&	0.4623&	0.5290&	16.07\\
				E3SC(5)&	0.7555&	0.7153&	0.8345&	31.52\\
				E3SC(6)&	0.6345&	0.6189&	0.6825&	18.69\\
SSCOMP(1)&0.4680&0.3939&0.5485&2.21\\
				SSCOMP(2)&0.3640&0.2913&0.4475&2.07\\
				SSCOMP(3)&0.2430&0.1536&0.4475&2.39\\
				SSCOMP(4)&0.3005&0.3191&0.3250&0.81\\
				SSCOMP(5)&0.5370&0.4911&0.6640&2.23\\
				SSCOMP(6)&0.3115&0.1586&0.3670&1.37\\							
				AMGL&	0.8460&	\textbf{0.8732}&	0.8710&	67.58\\
				MLRSSC&	0.7890&	0.7422&	0.8375&	52.44\\
				MSC\_IAS&0.7975&0.7732&0.8755&80.78\\
				LMVSC&	\textbf{0.9165}&	0.8443&	\textbf{0.9165}&	10.55\\

				\hline			
			\end{tabular}
			
	%\end{center}	
		}
	
\end{table}

\begin{table}			

\centering
			\caption{Clustering performance on Reuters data.\label{reu}}
			\resizebox{.9\columnwidth}{!}{
			\begin{tabular}{|c |c | c|c| c|  }
				\hline
				Method& Acc& NMI& Purity& Time (s)\\
				\hline
					SSCOMP(1)&0.2714&-&-&4721.90\\
				SSCOMP(2)&0.2730&-&-&5070.80\\
				SSCOMP(3)&0.2735&-&-&4524.10\\
				SSCOMP(4)&0.2719&-&-&4515.70\\
				SSCOMP(5)&0.3128&-&-&4162.40\\	
				AMGL&	0.1672&	-&	-&	32397\\	
						MSC\_IAS&0.5063 &0.2759 &0.6031 &7015.7 \\			
				LMVSC&	\textbf{0.5890}&	\textbf{0.3346}&	\textbf{0.6145}&	130.73\\

				\hline			
			\end{tabular}
			
	%\end{center}	
		}
	
\end{table}
\subsection{Experimental Results}
Tables \ref{cal7}-\ref{nus} show the clustering results of all methods on five data sets. Due to out-of-memory issue, many comparison methods, e.g., E3SC and MLRSSC, are not appliable to Reuters and NUS data sets which are more than 10,000 samples. Therefore, they are not listed in Table \ref{reu} and Table \ref{nus}. This demonstrates that the proposed method is low space complexity. In terms of accuracy, our method constantly outperforms others and the average improvement is at least 7\% over the second highest value. For NMI and purity, our method achieves comparable or even better performance than the other methods. This demonstrates the consistency and stability of our method. Taken Caltech7 and Handwritten as examples, we visualize their clustering results based on t-SNE technique in Figure \ref{tsne}. We can clearly observe the cluster pattern.

\begin{table}			
	\centering
			\caption{Clustering performance on NUS data.\label{nus}}
			\resizebox{.9\columnwidth}{!}{
			\begin{tabular}{|c |c | c|c| c|  }
				\hline
				Method& Acc& NMI& Purity& Time (s)\\
				\hline	
				SSCOMP(1)&0.1153&0.0755&0.1412&53.63\\
					SSCOMP(2)&0.1206&-&-&102.64\\
					SSCOMP(3)&0.1138&-&-&90.88\\
				SSCOMP(4)&0.0950&-&-&64.26\\
				SSCOMP(5)&0.0756&-&-&83.83\\	
					MSC\_IAS& 0.1548& \textbf{0.1521}& 0.1675& 45386\\
				LMVSC&\textbf{0.1553}&	0.1295&	\textbf{0.1982}&	165.39\\

				\hline			
			\end{tabular}
			
%	\end{center}	
	}
	
\end{table}

For running time comparison, our method can finish all data sets in several minutes. Our method is up to several orders of magnitude faster than other multi-view methods on large data sets. For example, MSC\_IAS consumes almost 13 hours to finish NUS data, while our method only takes 3 minutes. Note that the state-of-the-art scalable method SSCOMP takes extremely long time to run reuters data, which is due to the data's high dimensionality. According to our observation, the fluctuation of our execution time is mainly caused by the number of anchors. For instance, Handwritten achieves the best results when the fewest anchors are used, i.e., cluster number 10. Consequently, it takes the shortest time.

\begin{figure}[!htbp]
\centering
\subfloat[Caltech7]{\includegraphics[width=.4\textwidth]{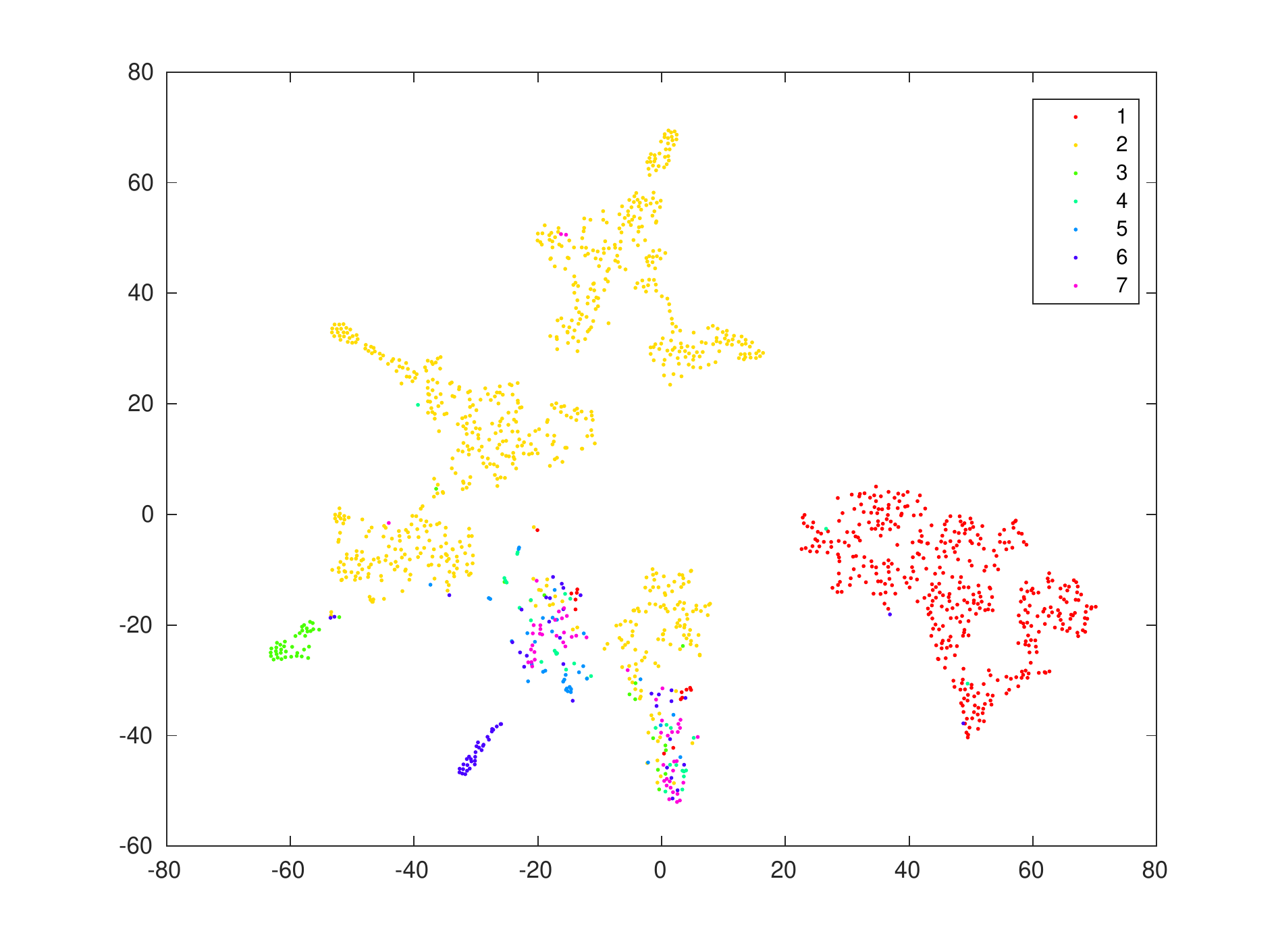}}\\
\subfloat[Handwritten]{\includegraphics[width=.4\textwidth]{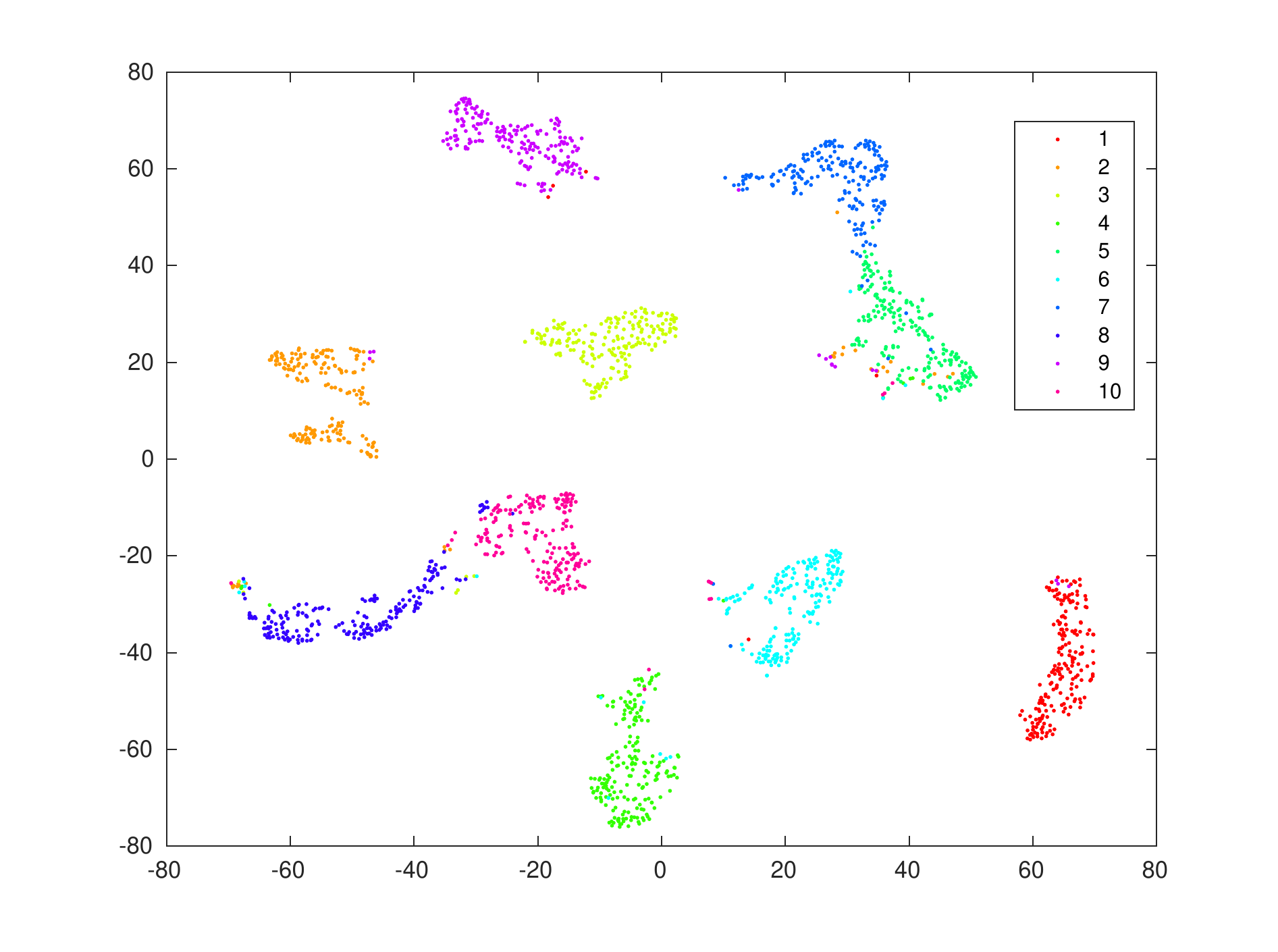}}\\
\caption{Visualization of some clustering results.  } \label{tsne}
\end{figure}
\subsection{Robustness Study}
To examine the robustness of our method, we perform additional experiments by adding three commonly seen kinds of noise. We choose the famous MNIST database, which consists of gray-scale images of size $28\times 28$. There are 70000 samples in total.  Since it is a single-view data, we construct a multi-view data set by adding different levels of noise to different views. Specifically, we first corrupt MNIST data by adding Gaussian noise with variance 0.01, 0.03, and 0.05; Salt\&Pepper noise with density 0.05, 0.1, and 0.2; Speckle noise with variance 0.05, 0.1, and 0.15, respectively. Then, we concatenate them to form three three-view data sets. Some example images are shown in Figures \ref{noiseg}-\ref{noisesa}. For these large data sets, we find that only SSCOMP and our method can handle it, while other recent multi-view clustering methods fail. 

According to the results in Table \ref{mnist}, our method obtains better performance than SSCOMP in all metrics. This also verifies the advantage of multi-view learning which can exploit the complementarity of multi-view data. In terms of computation time, our method is 20 times faster than SSCOMP.

\begin{table}[t]%[!htbp]
\centering

\caption{Clustering performance on MNIST data.\label{mnist}}
%\scalebox{0.8}{
\resizebox{1\columnwidth}{!}{
\begin{tabular}{|c|c|c|c|c|c|}
%\Xhline{1.0pt}
\hline
{Noise}&{Method} & Acc& NMI& Purity& Time (s)\\
\hline

\multirow{4}*{Gaussian}&SSCOMP(1)&0.4300&0.4726&0.5965&1120.80\\
\cline{2-6}
&SSCOMP(2)&0.4465&0.4749&0.5908&1151.50\\
\cline{2-6}
&SSCOMP(3)&0.4418&0.4754&0.5920&1130.30\\
\cline{2-6}
&LMVSC&\textbf{0.5565}&\textbf{0.5096}&\textbf{0.6282}&55.17	\\
\hline

\multirow{4}*{Speckle}&SSCOMP(1)&0.4417&0.4770&0.5967&1314.90\\
\cline{2-6}
&SSCOMP(2)&0.4560&0.4795&0.5995&1287.70\\
\cline{2-6}
&SSCOMP(3)&0.4556&0.4855&0.6043&1303.40\\
\cline{2-6}
&LMVSC&\textbf{0.5920}&\textbf{0.5178}&\textbf{0.6183}&66.01	\\
\hline
\multirow{4}*{Salt\&Pepper}&SSCOMP(1)&0.4536&0.4782&0.5964&1181.90\\
\cline{2-6}
&SSCOMP(2)&0.4513&0.4822&0.6021&1306.10\\
\cline{2-6}
&SSCOMP(3)&0.4816&0.4923&0.6307&1302.00\\
\cline{2-6}
&LMVSC&\textbf{0.5889}&\textbf{0.5374}&\textbf{0.6598}&73.33	\\
\hline
\end{tabular}

}
			
%	\end{center}
	
\end{table}

\begin{figure}[!htbp]
\centering
\includegraphics[width=.3\textwidth, angle =270]{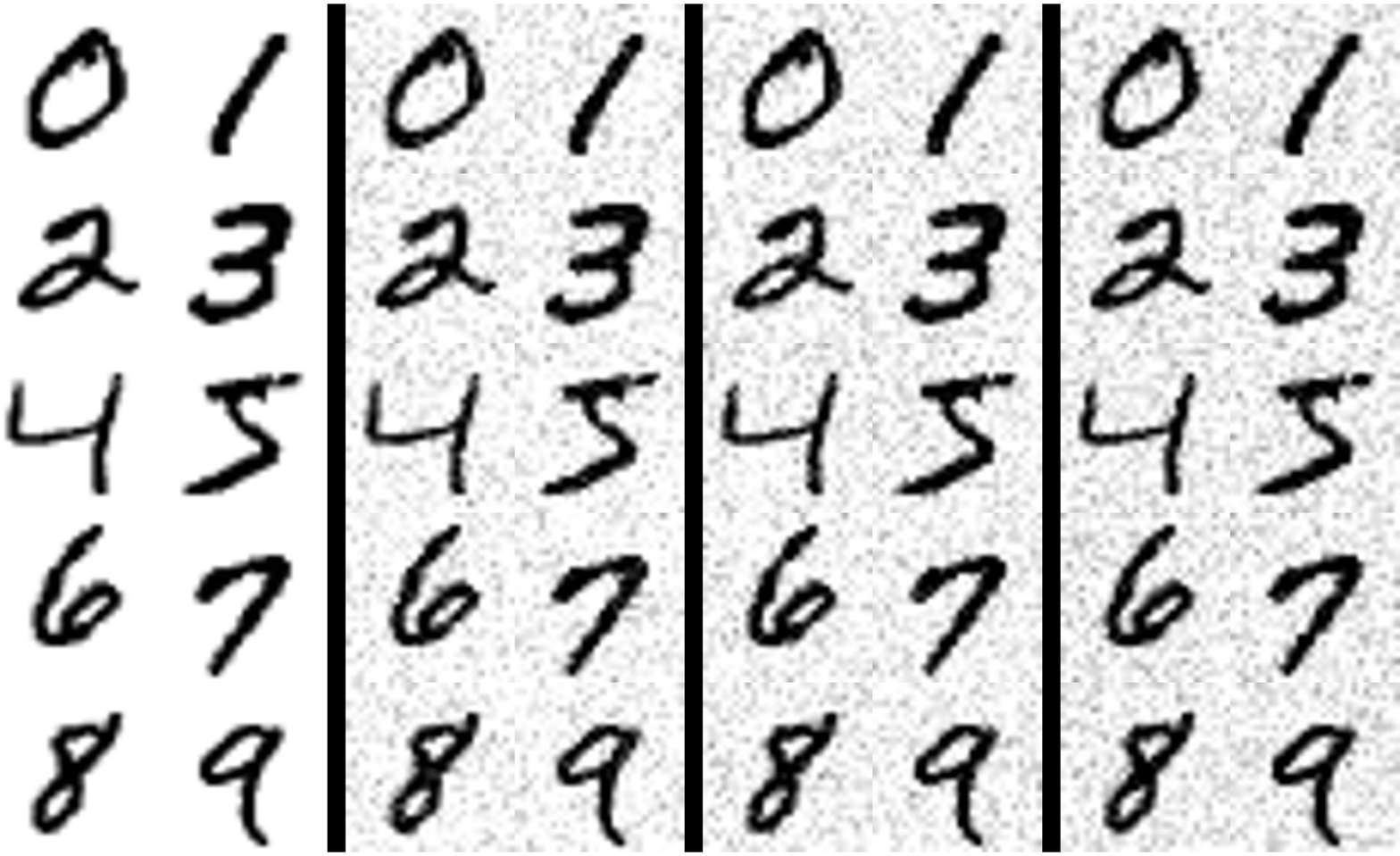}
\caption{Some sample images of MNIST are shown in the 1st column. Corresponding noisy images contaminated by Gaussian noise with variance 0.01, 0.03, and 0.05, are displayed in the 2nd, 3rd, and 4th column, respectively.  } \label{noiseg}
\end{figure}
\begin{figure}[!htbp]
\centering
\includegraphics[width=.3\textwidth,angle =270]{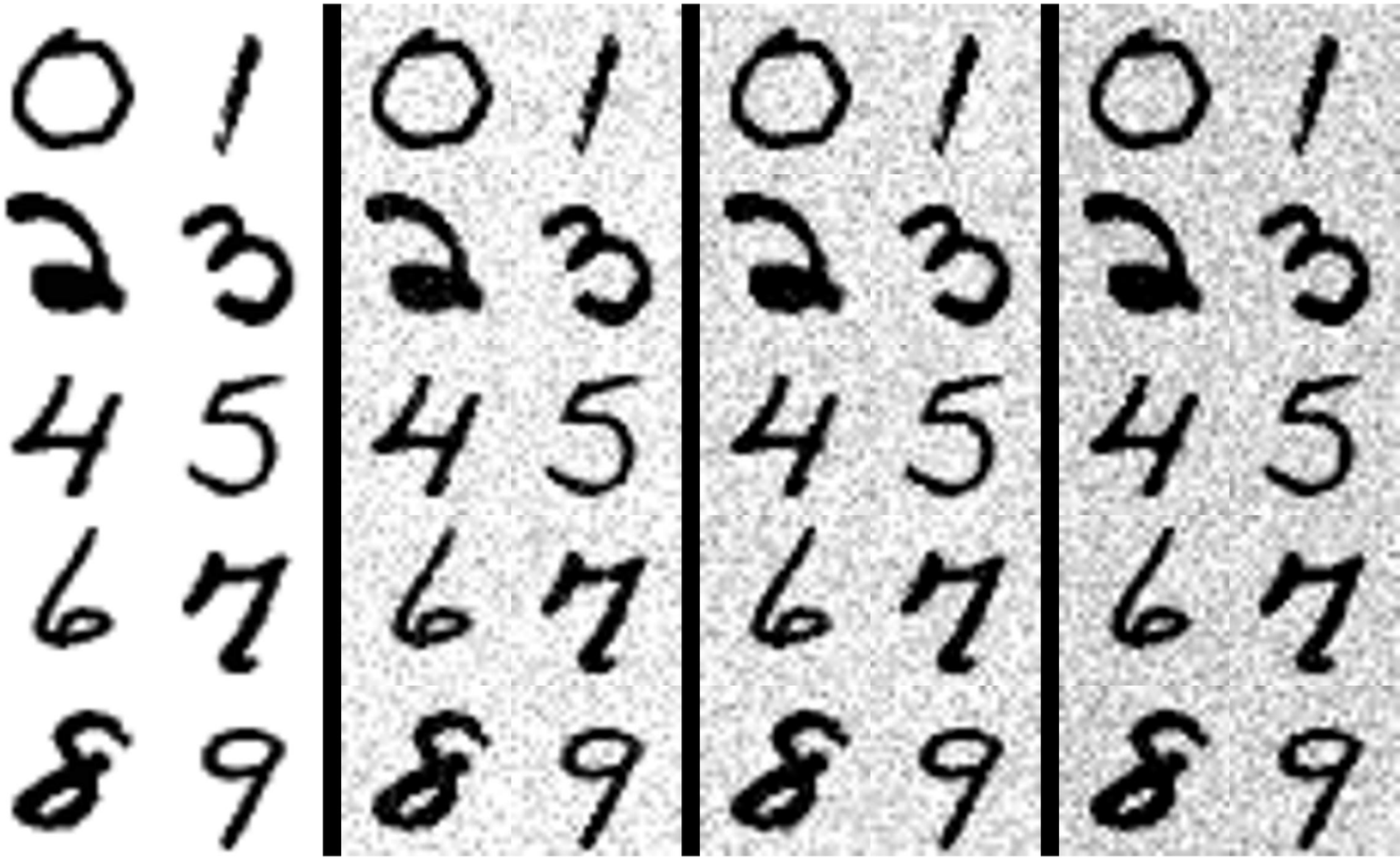}
\caption{Some sample images of MNIST are shown in the 1st column. Corresponding noisy images contaminated by Speckle noise with variance 0.05, 0.1, and 0.15, are displayed in the 2nd, 3rd, and 4th column, respectively.  } \label{noisesp}
\end{figure}

\subsection{Parameter Analysis}
During the experiments, we tune two parameters based on grid search, i.e., the number of anchors and $\alpha$. Taking Handwritten and MNIST data sets as examples, we show our model's sensitivity to their values in Figures \ref{fig1} and \ref{fig2}. We can see that $\alpha$ tends to have a small value since the performance is degraded when $\alpha$ increases. We have similar observation for anchor number. Too many anchors will make themselves less representative and introduce extra errors. Consequently, the performance will be deteriorated.
\begin{figure}[!htbp]
\centering
\includegraphics[width=.3\textwidth, angle =270]{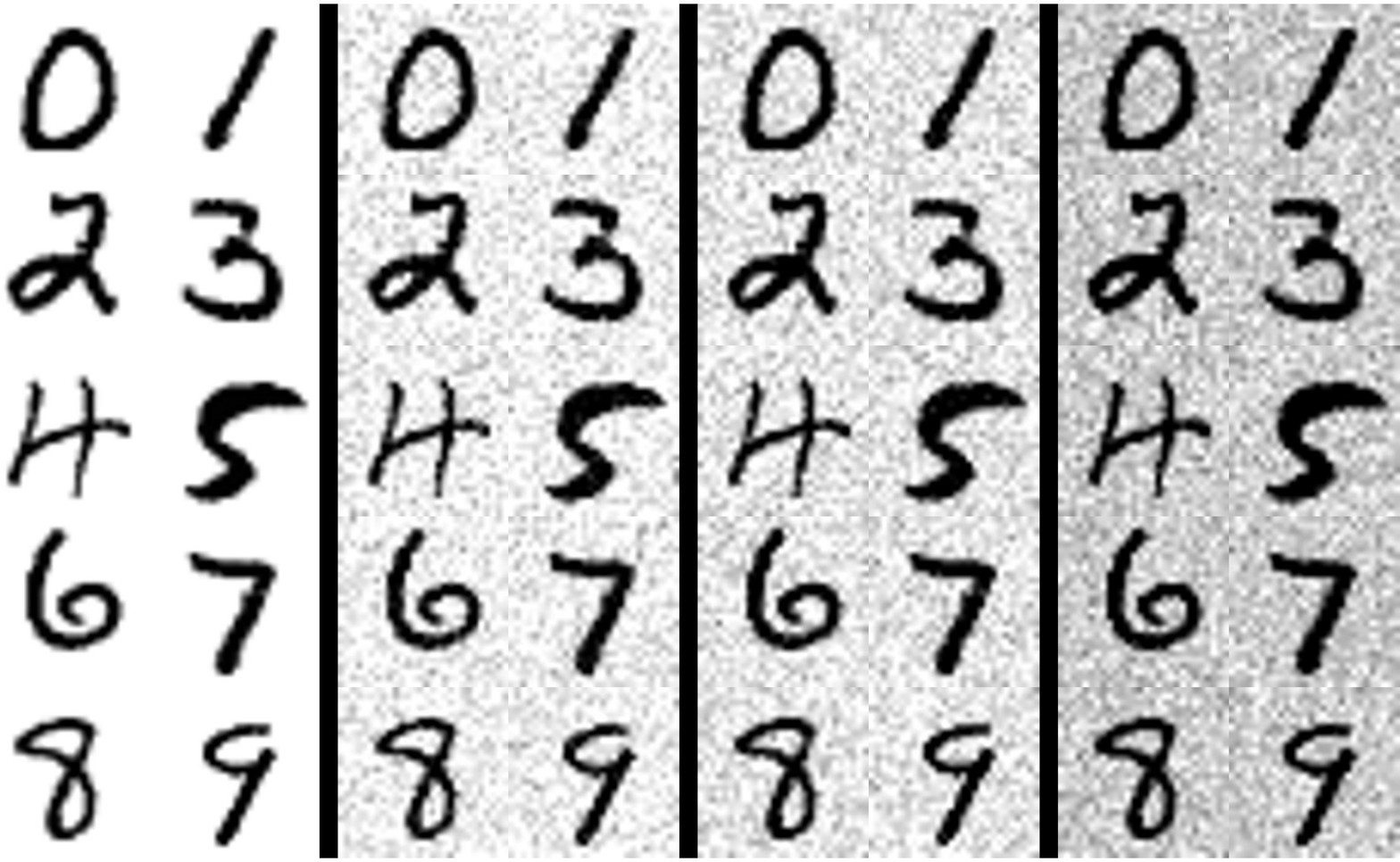}
\caption{Some sample images of MNIST are shown in the 1st column. Corresponding noisy images contaminated by Salt\&Pepper noise with density 0.05, 0.1, and 0.2, are displayed in the 2nd, 3rd, and 4th column, respectively.  } \label{noisesa}
\end{figure}

  \begin{table}[!hbtp]
\begin{center}

\caption{Description of single-view data sets. \label{singledata}}

\begin{tabular}{l|clll}
\hline%{1.0pt}%

{Data Sets} &{\#Feature} & {\#Classes}  & {\#Sample}  \\\hline
RCV1& 1979& 103& 193844\\
CoverType&54 & 7& 581012\\
%Type & image  & image & image \\
%\Xhline{1.0pt}%
\hline
\end{tabular}%}
\end{center}
\end{table}

 \begin{figure*}[!htbp]
\centering
\includegraphics[width=.33\textwidth]{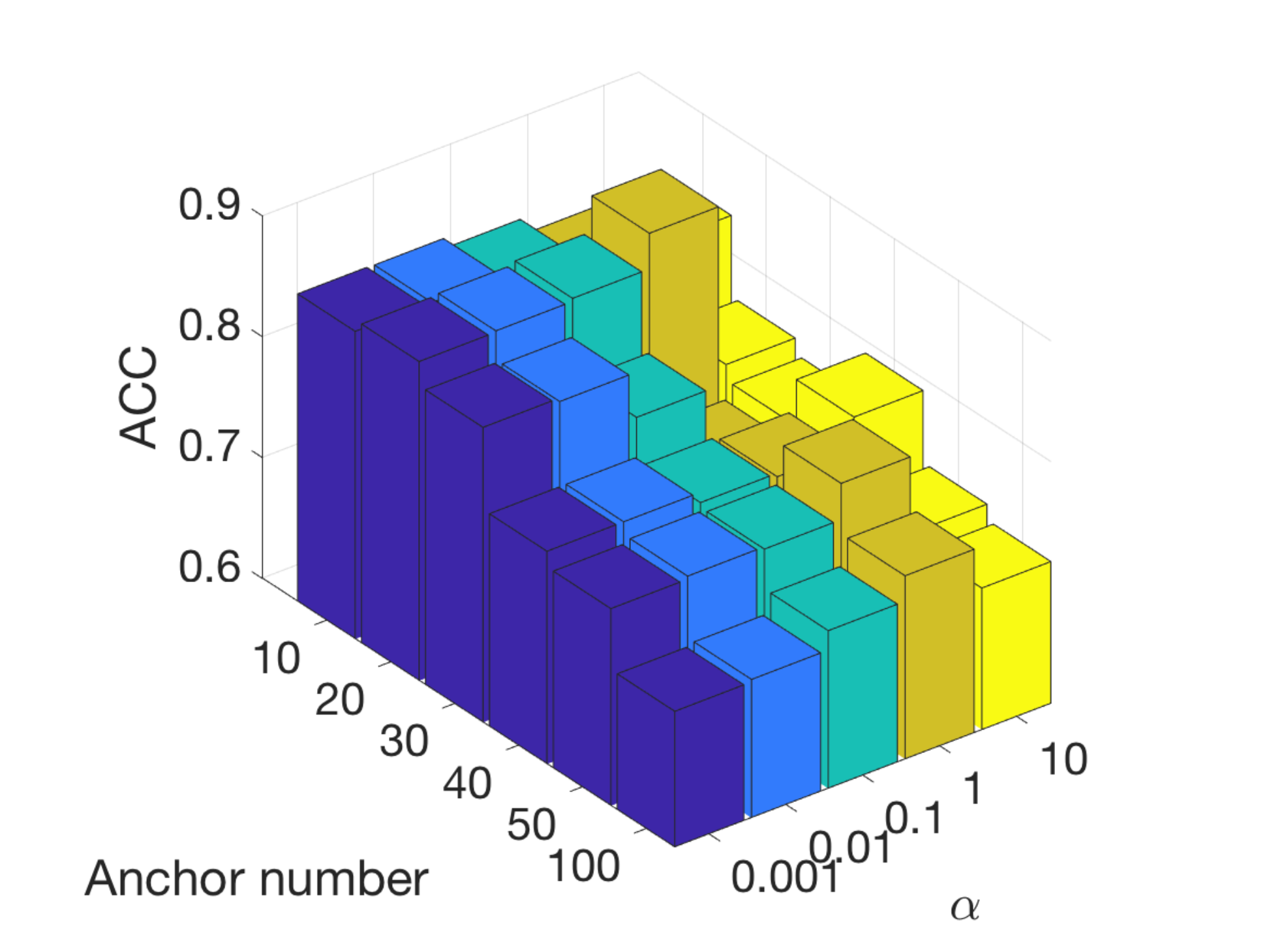}
\hspace{-1.cm}
\includegraphics[width=.33\textwidth]{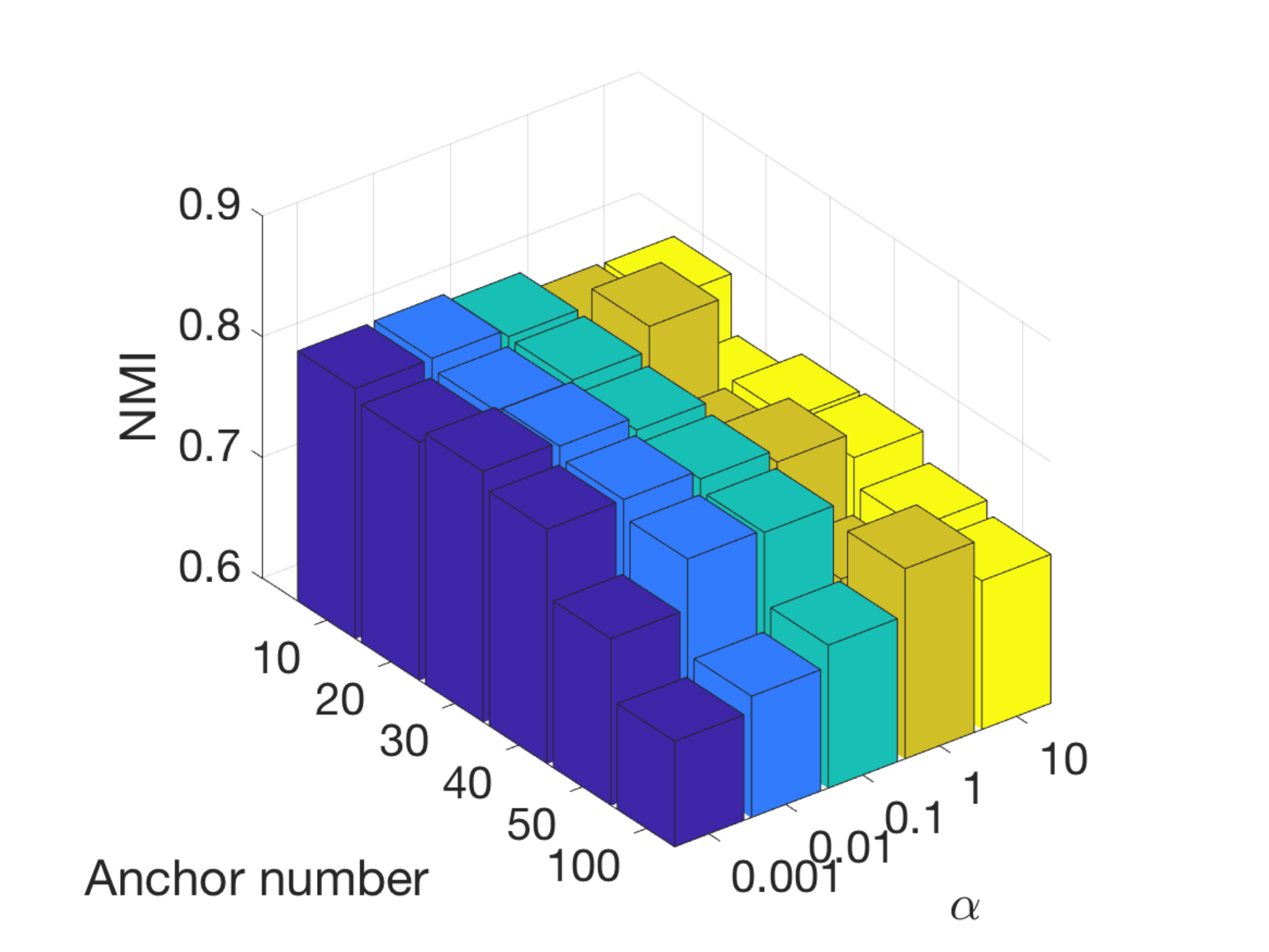}
\hspace{-1.cm}
\includegraphics[width=.33\textwidth]{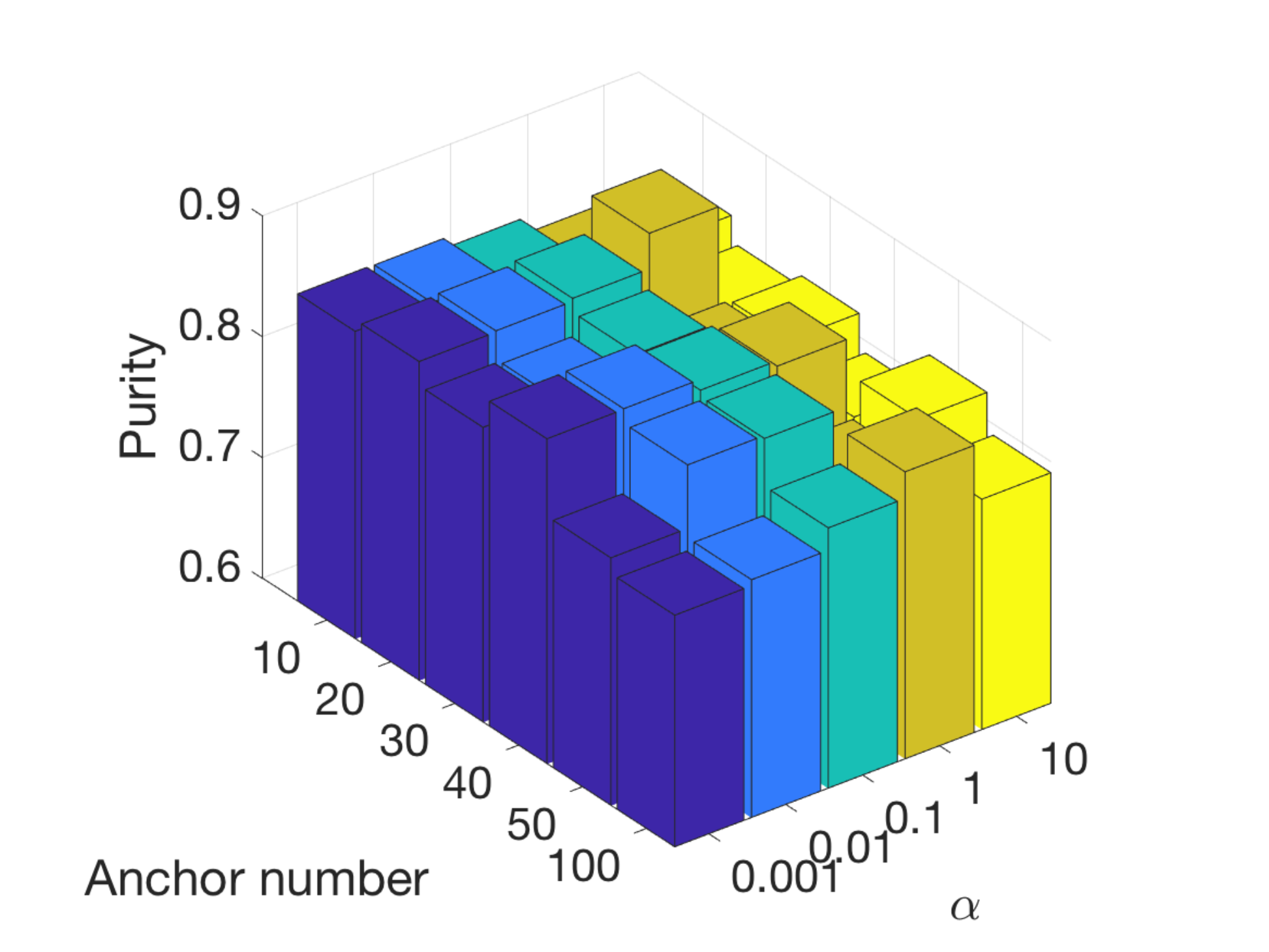}
\caption{Sensitivity analysis of parameters for our method over Handwritten data set.} \label{fig1}
\end{figure*}

 \begin{figure*}[!htbp]
\centering
\includegraphics[width=.33\textwidth]{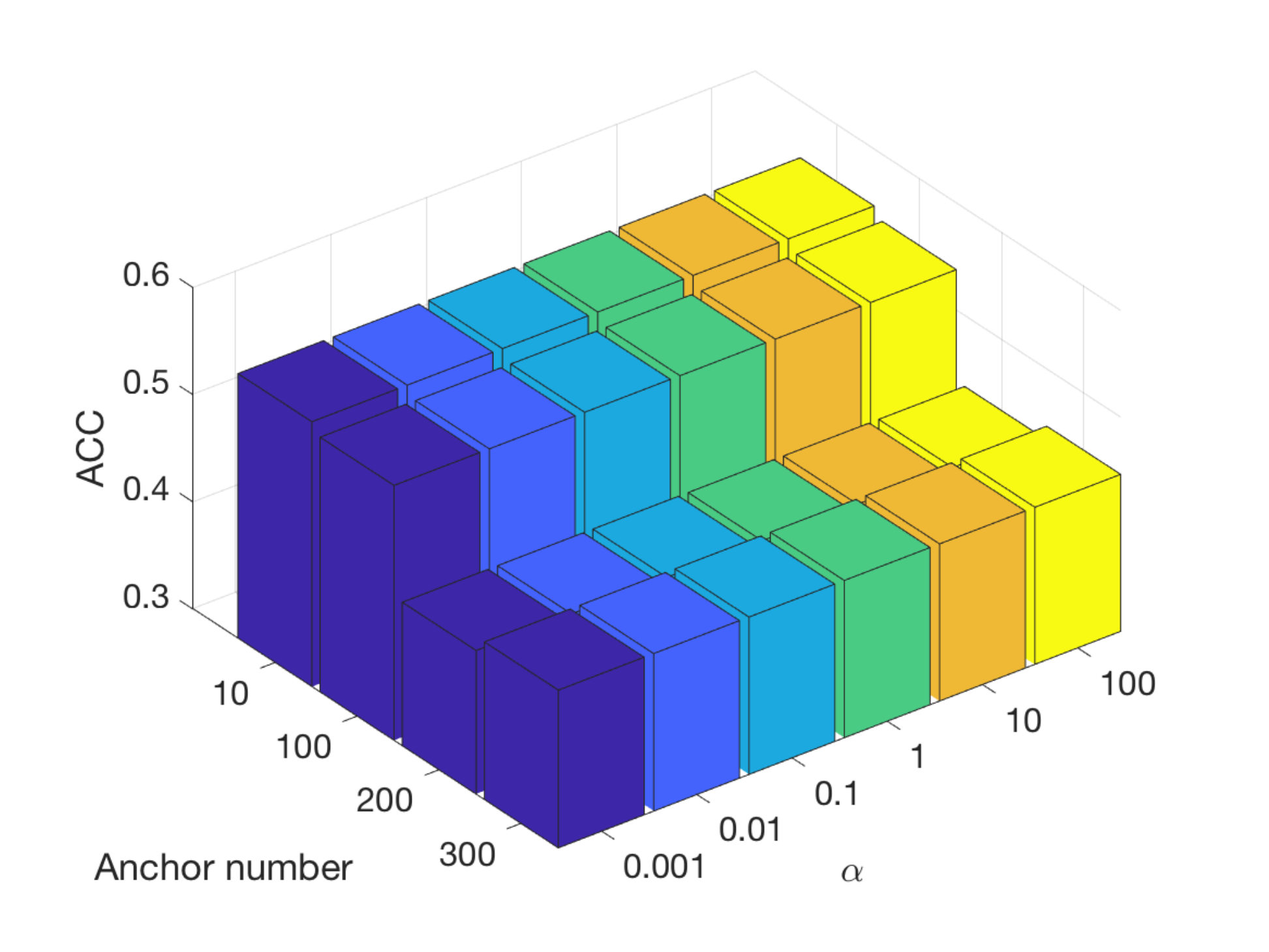}
\hspace{-.5cm}
\includegraphics[width=.33\textwidth]{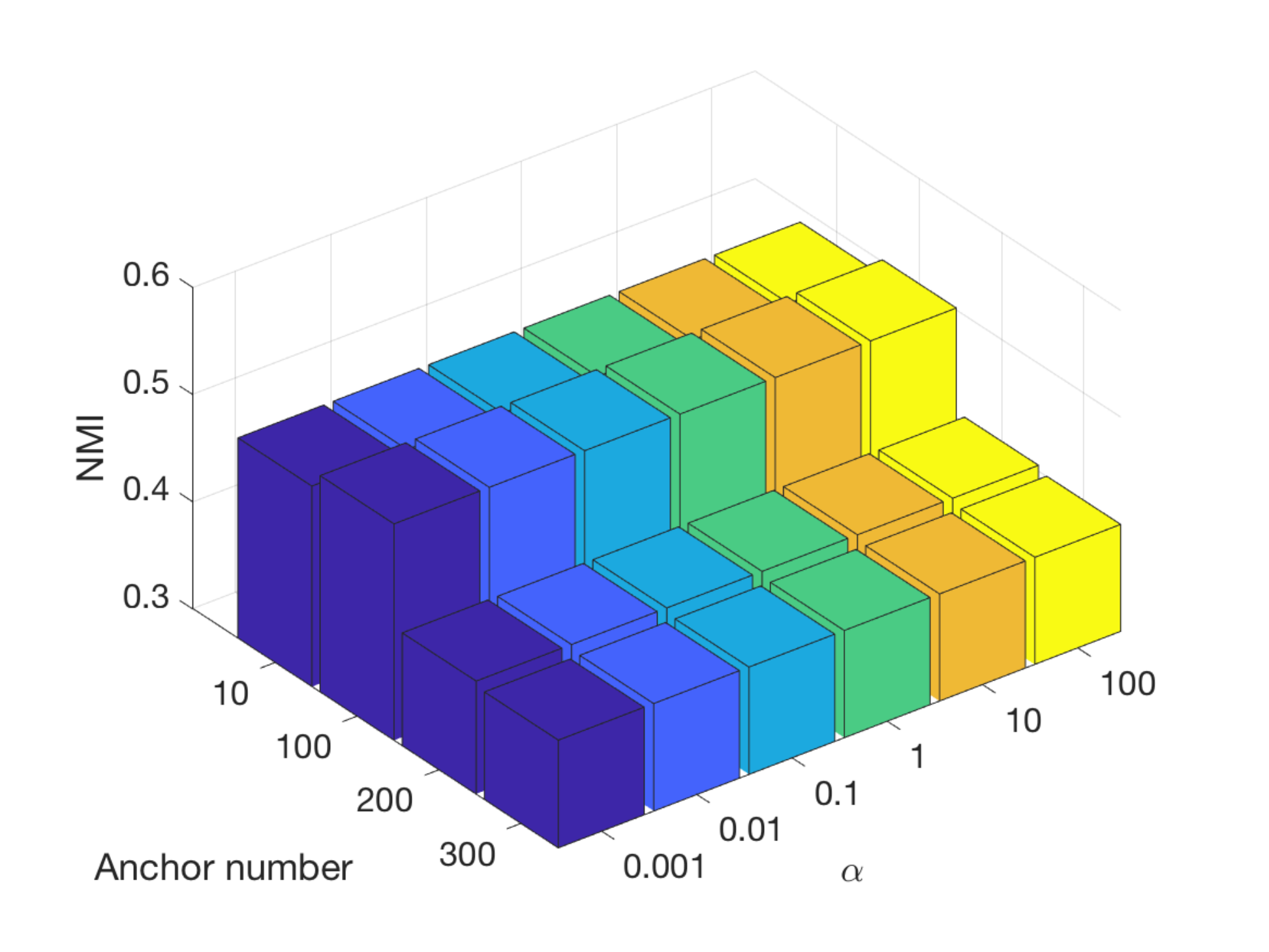}
\hspace{-.5cm}
\includegraphics[width=.33\textwidth]{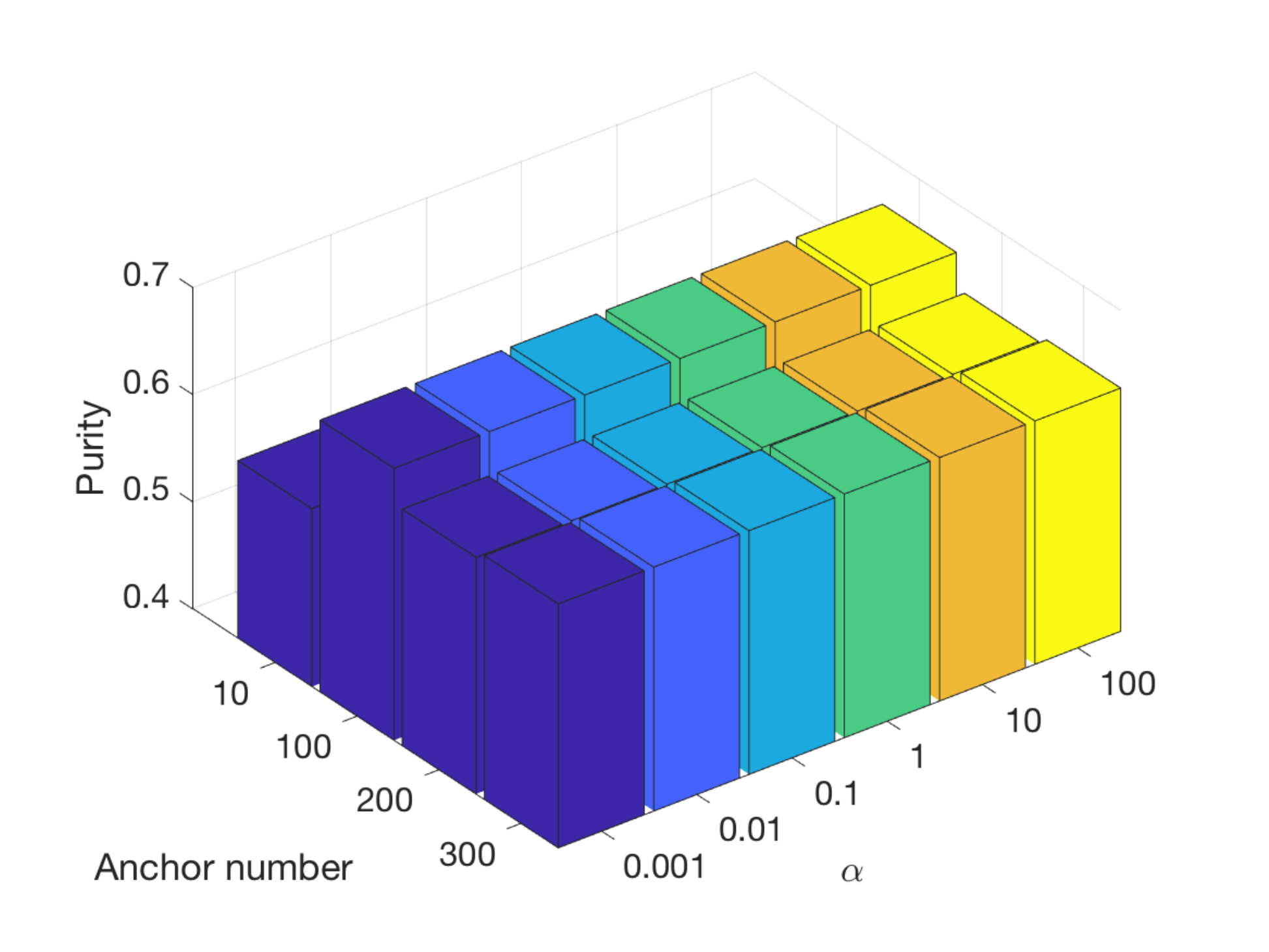}
\caption{Sensitivity analysis of parameters for our method over MNIST data set.} \label{fig2}
\end{figure*}
  \begin{table}[!hbtp]
\centering
\caption{Results of single-view data sets. \label{singleres}}

\resizebox{1\columnwidth}{!}{
\begin{tabular}{|c|l|c|l|c|l|c|l|c|l}
\hline%{1.0pt}%
%\hline
{Data Sets} &{Method} & {Acc} & {NMI} & {Purity}  & {Time (s)}   \\
\hline
\multirow{2}*{RCV1}&KM&0.1846&0.2447&0.2550&1569.07\\
\cline{2-6}
&LMVSC& 0.1929& 0.2600& 0.2985& 1035\\
\hline
\multirow{2}*{CoverType}&KM& 0.2505&0.0617&0.3039& 156.72 \\
\cline{2-6}
&LMVSC& 0.3862& 0.1273& 0.4889&2755.8\\
\hline

\end{tabular}}
%\end{center}
\end{table}
\section{Experiment on Single-View Data}
We use two large-scale single-view data sets to demonstrate the performance of our algorithm on single-view data. Specifically, $\textbf{RCV1}$ consists of newswire stories from Reuters Ltd. $\textbf{CovType}$ contains instances for predicting forest cover type from cartographic variables. Their statistics are summarized in Table \ref{singledata}. We found that SSCOMP can not finish within 24 hours, while many other large-scale subspace clustering \cite{fan2018accelerated,xiao2015falrr} methods run out of memory. For comparison, we include $k$-means (KM) as baseline. From Table \ref{singleres}, we can see that our method improves the performance significantly. Considering the size of data, our computation time is still reasonable. For CoverType, KM runs fast since the cluster number is quite small.

\section{Conclusion}
In this paper, we propose a novel multi-view subspace clustering algorithm. To the best of our knowledge, LMVSC is the very first effort of addressing the large-scale problem of multi-view subspace clustering. Our proposed method has a linear computation complexity, while maintaining high level of clustering accuracy. Specifically, for each view, one smaller graph is built between the raw data points and the generated anchors. Then, a novel integration mechanism is designed to merge those graphs, so that the eigen decomposition can be accelerated significantly. Furthermore, our proposed method also applies to single-view data. Extensive experiments on real data sets verify the effectiveness, efficiency, and robustness of the proposed method against other state-of-the-art techniques.
\section{ACKNOWLEDGMENT}
This paper was in part supported by Grants from the Natural
Science Foundation of China (Nos. 61806045 and 61572111) and
Fundamental Research Fund for the Central Universities of
China (No. ZYGX2017KYQD177).
%\vfill\eject
%\begin{quote}
%\begin{small}
%\textbackslash 
\bibliographystyle{aaai}
%\textbackslash 
\bibliography{ref}
%\end{small}
%\end{quote}
\end{document}